%% file: paper.tex
\documentclass[]{llncs}
\usepackage{misc0}

 \usepackage{algorithm}
\usepackage{algpseudocode}
 \usepackage{mathtools}
 \usepackage{color}
 \usepackage{microtype}
\usepackage{thm-restate} 
\newcommand{\agents}{\ensuremath{\mathbf{A}}\xspace} 
\newcommand{\elaxiom}{\ensuremath{{\sf a}}\xspace} 
\newcommand{\K}{\ensuremath{\mathbf{K}}\xspace} 
 
\newcommand{\ELK} {\ensuremath{\mathcal{ELK} }\xspace} 
\newcommand{\LK} {\ensuremath{\mathcal{LK} }\xspace} 
 
\newcommand{\prop}[1]{\ensuremath{#1_{\sf prop} }\xspace} 
\newcommand{\ELrl}{\ensuremath{\EL_{\sf lhs,rhs}}\xspace}
\newcommand{\EXoracle}{\ensuremath{ {\sf EX}_{\Fmf(\EL_{{\sf lhs}, {\sf rhs}},\agents),\Omc^k}}\xspace}
\newcommand{\consistent}[1]{\ensuremath{#1\text{-consistent}}\xspace}

\begin{document}
\pagestyle{headings}  

\title{Learning Ontologies with Epistemic Reasoning: \\
The \EL Case}
\author{Ana Ozaki and  Nicolas Troquard
}
\institute{KRDB Research Centre, Free University of Bozen-Bolzano, Italy}
\maketitle  

\begin{abstract}
We investigate the problem of learning description logic ontologies \emph{from entailments} 
 via queries,  using epistemic reasoning. 
We introduce a new learning model consisting of 
 \emph{epistemic membership and example queries}   and show that 
 polynomial learnability in this  model coincides with 
polynomial learnability in  Angluin's exact learning model 
with membership and equivalence queries. 
We then instantiate our learning framework to \EL  and show some complexity 
results for an epistemic extension of \EL where   epistemic operators 
can be applied over the axioms.
Finally, we transfer 
known results for \EL ontologies and  its fragments
 to our 
learning model based on epistemic reasoning.   
\end{abstract}

\input{introduction.tex}
\input{learningFramework.tex}
\input{epistemicEL.tex}
\input{learningEL.tex}

\input{conclusion.tex}

\newcommand{\SortNoop}[1]{}
\bibliography{references.bib,languages.bib}
\bibliographystyle{abbrv}
\appendix

\input{appendix}

\end{document}

%% file: introduction.tex
\section{Introduction}

Description logics (DL) balance expressivity and complexity of reasoning, 
resulting in a family of formalisms which can   capture 
conceptual knowledge in various domains~\cite{dlhandbook}. 
One of the most popular ontology languages, featuring polynomial time 
complexity of reasoning tasks such as entailment, is \EL~\cite{BBL-EL}, which 
allows conjunctions ($\sqcap$) and existential restrictions ($\exists$) in its concept 
expressions but disallows negations of concepts. 
The following example illustrates \EL ontologies (Section~\ref{sec:reasoning}) representing 
knowledge of   experts in different domains.
\begin{example}\label{ex:1}
${\sf Ana}$ knows 
about  ${\sf Brazilian\ music} $ $({\sf BM})$ and ${\sf Nicolas}$ is an expert in 
${\sf French}$ ${\sf cuisine}$ $ ({\sf FC})$.
We can represent some   parts
of their knowledge as follows. 
 \[
\begin{array}{ll@{\hspace*{1.5em}}ll}
  &\Omc^{\sf BM}_{\sf Ana}= \{{\sf BrazilianSinger(Caetano)} &\Omc^{\sf FC}_{\sf Nicolas}=\{{\sf FrenchChef(Soyer)}&\\  
  
  &{\sf BossaNova} \sqsubseteq {\sf BrazilianMusicStyle} 
  & 
   
      {\sf Crepe} \sqsubseteq \exists {\sf contains. Flour} &\\
  &{\sf ViolaBuriti} \sqsubseteq \exists {\sf madeFrom. Buriti}\}
    &   {\sf Crepe} \sqcap \exists {\sf contains. Sugar} \sqsubseteq {\sf Dessert}\}&
 \end{array}
\]
\end{example}
 
Naturally, domain experts---humans, or artificial entities with complex neural 
networks---cannot be expected to  be able to   easily transfer their knowledge. 
However, when specific questions 
about the domain are posed, e.g., `is Bossa Nova a Brazilian music style?', 
an expert in the domain of Brazilian music can  accurately decide whether such statement holds or not. 
So the ontology representation of the  knowledge of 
an expert, even though \emph{not directly accessible}, 
  can be learned via a trial and error process in which
  individuals or machines, generically   called   \emph{agents}, 
  communicate with each other,    in order to learn from the other agents. 
  We assume that the \emph{target} domain of interest to be learned
  is  represented by a logical theory formulated in an ontology language.
  
  %



%
In computational learning theory, 
a classical communication protocol coming from the exact learning model~\cite{angluinqueries} 
is based on questions of two types: \emph{membership} and \emph{equivalence} queries. 
In a learning from entailments setting~\cite{DBLP:conf/icml/FrazierP93}, 
these questions can be described as follows. 
Membership queries correspond to asking whether a certain statement formulated as a logical 
sentence follows from the target. 
Equivalence queries correspond to asking whether a 
certain logical theory, called \emph{hypothesis},  
precisely describes the target.
If  there are  wrong or missing statements in the hypothesis, 
a statement illustrating the imprecision should be returned 
to the agent playing the role of the \emph{learner}. 
\begin{example}\label{ex:2} 
Assume ${\sf Ana}$ wants to learn about ${\sf French}$ ${\sf cuisine}$.
She asks ${\sf Nicolas}$ whether it follows from 
his knowledge that `every crepe is a dessert', in symbols, 
`does $\Omc^{\sf FC}_{\sf Nicolas}\models  {\sf Crepe}   \sqsubseteq {\sf Dessert}$?',  
which the answer in this case is `no', since only those which contain sugar are 
considered desserts. 
To receive new statements about ${\sf French}$ ${\sf cuisine}$ from ${\sf Nicolas}$, 
${\sf Ana}$ needs to pose equivalence queries, in symbols `does 
$\Omc^{\sf FC}_{\sf Ana}\equiv\Omc^{\sf FC}_{\sf Nicolas}$?'. Each time she 
poses this type of questions, her best interest is to tell him everything she knows about 
 ${\sf French}$ ${\sf cuisine}$.
\end{example}
One of the main difficulties in implementing   
this protocol in practice~\cite[page 297]{Mohri:2012:FML:2371238} comes from
the putative unreasonableness of equivalence queries.
Whenever a learner poses an equivalence query, the expert playing 
the role of an \emph{oracle} needs to 
evaluate the whole hypothesis and decide whether or not it is  equivalent
to the target. 
If not, then the oracle 
returns
a statement in the logical difference between 
the hypothesis and the target.
One way out of this difficulty is hinted to us by a simple observation:
during interactive communication among agents, 
not only domain knowledge is exchanged and acquired
but also second-order knowledge, which is the knowledge of what is known by the other agents.

\begin{example} 
 When ${\sf Ana}$ and ${\sf Nicolas}$ 
 communicate, they know what they have already told to each other.  
 If ${\sf Ana}$  tells ${\sf Nicolas}$ that 
 `Buriti is a Brazilian tree'
 (${\sf Nicolas}$ now \emph{knows} this statement, in symbols, $\K_{\sf Nicolas}({\sf Buriti}\sqsubseteq {\sf BrazilianTree})$)
 and that `Viola de Buriti is made from Buriti' 
 ($\K_{\sf Nicolas}({\sf ViolaBuriti}\sqsubseteq \linebreak \exists {\sf madeFrom}.{\sf Buriti})$)
 she does not need to 
 tell him that `Viola de Buriti is made from a Brazilian tree'
 (as it follows that 
 $\K_{\sf Nicolas}({\sf ViolaBuriti}\sqsubseteq \linebreak \exists {\sf madeFrom}.{\sf BrazilianTree})$, 
 see Section~\ref{sec:reasoning}).
\end{example}

In this paper, we thus propose a new and more realistic learning model.
It is based on a protocol which takes into account what is known by the agents, either 
because a statement was explicitly communicated or because it is a logical consequence of 
previous statements given during their interaction.
Our protocol is based on queries of two types. 
The first is an epistemic version of membership queries where 
the oracle `remembers' those membership queries whose reply 
was `yes'.
We call the second type \emph{example queries}. When asked an example query,
the oracle answers a statement which follows from its knowledge but does not follow 
from its knowledge about what the learner knows. The oracle also `remembers' that 
the statements given are now known by the learner. 

The first contribution of this work is the introduction of the learning 
model based on epistemic reasoning, which we call \emph{epistemic learning model}, and
an analysis of its `power' in comparison with the exact learning model (Figure~\ref{fig:models}).
The second   is an instantiation to the \EL ontology language, whose polynomial learnability 
has 
been investigated in the exact learning model~\cite{KLOW18,DBLP:conf/kr/DuarteKO18,DBLP:conf/aaai/KonevOW16}. 
%

In more details, the epistemic learning model is introduced in Section~\ref{sec:learning-epistemic}.
We then establish in Section~\ref{sec:epiVSexac} that polynomial learnability is strictly harder  in the 
epistemic model without (an epistemic version of) membership queries (Theorems~\ref{thm:onedirection} and~\ref{thm:harder}).
Nonetheless, it coincides  with
 polynomial learnability in the exact learning model if both 
 types of queries are allowed (Theorem~\ref{thm:transfer}). 
Since it is known that polynomial learnability 
in the exact learning model with only equivalence queries implies polynomial learnability 
in the classical probably approximately correct learning model (PAC)~\cite{angluinqueries,Valiant}, 
it follows that polynomial learnability in the epistemic learning model 
with only example queries implies polynomial learnability 
in the PAC learning model.
The same relationship holds for the case where we have (an epistemic version of) membership queries in the epistemic model and 
the PAC model also allows
  membership queries. 
\input{figure.tex}
We also show in Section~\ref{sec:reasoning} some complexity results for an epistemic extension of \EL, 
which we call \ELK. 
In particular, we show that satisfiability in \ELK, which 
includes Boolean combinations of \EL axioms, does increase 
the \NP-completeness of propositional logic (Theorem~\ref{th:elk}).
We then show that 
a fragment of \ELK 
features  
\PTime 
complexity for the satisfiability and entailment problems
 (Theorem~\ref{th:conj-elk}), as in \EL~\cite{BBL-EL}. Crucially, it captures 
the epistemic reasoning that
  the agent playing the role of the oracle needs to perform.
Finally, in Section~\ref{sec:learningEL} we transfer known results~\cite{KLOW18,DBLP:conf/kr/DuarteKO18} for  \EL  
in the exact learning model to the epistemic learning model.  
  





%% file: figure.tex
\begin{figure}[t]
  \centering
  \[
  \begin{array}{lcccr}
    \text{{\scriptsize EPISTEMIC[{\sf MEM},{\sf EX}]}~~~~~} & \stackrel{\text{Th.}~\ref{thm:transfer}}{\large{=}} & \text{~~~~~{\scriptsize EXACT[{\sf MEM},{\sf EQ}]}~~~~~} & \stackrel{\text{\cite{Mohri:2012:FML:2371238} and \cite{Blum:1994:SDM:196751.196815}}}{\large{\subset}} & \text{~~~~~{\scriptsize PAC[{\sf MEM}]}}\\
    \\
    \text{{\scriptsize EPISTEMIC[{\sf EX}]}} & \stackrel{\text{Th.~\ref{thm:onedirection} and \text{Th.}~\ref{thm:harder}}}{\large{\subset}} & \text{{\scriptsize EXACT[{\sf EQ}]}} & \stackrel{\text{\cite{angluinqueries} and \cite{Blum:1994:SDM:196751.196815}}}{\large{\subset}} & \text{{\scriptsize PAC}}\\
  \end{array}
  \]
  
  \caption{Polynomial learnability. Each class denotes the set of frameworks that are polynomial query learnable in the corresponding learning model. {\sf MEM}, {\sf EQ} and {\sf EX} stand for membership, equivalence, and example queries respectively.
  }
\label{fig:models}
\end{figure}

%% file: learningFramework.tex
\section{Learning with epistemic reasoning}\label{sec:learning-epistemic}

We first define the epistemic extension of a   
 description logic \Lmc,  which is often a notation variant of a fragment of  
   first-order logic or propositional logic.  
The epistemic extension of   \Lmc allows 
expressions of the form
 `agent $i$ knows some axiom of \Lmc'. 
 We then use the epistemic extension of a logic to define 
 a learning framework based on epistemic reasoning. 
\subsection{The epistemic extension of \Lmc}\label{subsec:extension}
 
In the following, we   formalise the epistemic extension \LK  of   a description logic  \Lmc.  
Our notation and definitions can be easily adapted to the case 
\Lmc is a (fragment of) first-order or propositional logic.  
Assume symbols of \Lmc are taken from pairwise disjoint and countably infinite sets of concept, role  
and individual names  
\NC, \NR and \NI, respectively. 
Let $\agents$ be a set of agents. 
An \LK  axiom is an expression of the form:
$
\beta ::= \alpha \mid  \K_i \beta   
$ where $\alpha$ is an \Lmc formula and $i\in\agents$. 
\LK  formulas $\varphi,\psi$ are expressions of the form: 
$
 \varphi ::=  \beta\mid   \neg\varphi\mid    \varphi\wedge \psi 
$
where $\beta$ is an \LK  axiom.


An \Lmc  \emph{interpretation} $\Imc=(\Delta^\Imc,\cdot^{\Imc})$ over a non-empty 
set $\Delta^\Imc$, called the \emph{domain}, defines an \emph{interpretation 
function}~$\cdot^{\Imc}$ that maps each 
concept name $A  \in \NC $ to a subset~$A^{\Imc} $ of~$\Delta^\Imc$, each role
name $r  \in \NR $ to a binary relation~$r^\Imc $ on~$\Delta^\Imc$, and each 
individual name $a\in\NI$ to an element $a^\Imc\in \Delta^\Imc$. 
The extension of the mapping~$\cdot^{\Imc}$ from concept names to \Lmc  
complex concept expressions depends on the precise definition of \Lmc.
We write $\models_{\Lmc}$ and $\equiv_\Lmc$ 
to denote the entailment and equivalence relations for \Lmc formulas, respectively. 
    
An \LK interpretation $\Imf=(\Wmc,\{\Rmc_i\}_{i\in \agents})$ consists of a set 
$\Wmc$ of \Lmc interpretations 
and a set of accessibility relations $\Rmc_i$ on $\Wmc$, 
one for each agent $i\in \agents$. We assume that the relations $\Rmc_i$ are 
equivalence relations. 
A pointed \LK interpretation is a pair $(\Imf,\Imc)$ where 
$\Imf=(\Wmc,\{\Rmc_i\}_{i\in \agents})$ is an \LK interpretation and 
$\Imc$ is an element of \Wmc. 
The \emph{entailment} relation $\models_\LK$ of an \LK formula~$\varphi$ 
in~$\Imf=(\Wmc,\{\Rmc_i\}_{i\in \agents})$  pointed at   $\Imc\in\Wmc$
 is inductively defined (for simplicity, we may omit the subscript $_\LK$ from  $\models_\LK$):
\[
\begin{array}{ll@{\hspace*{1.5em}}ll}
    
  \Imf,\Imc\models \alpha & \text{iff }  \Imc\models_{\Lmc} \alpha 
    & \Imf,\Imc\models \phi \wedge \psi & \text{iff } \Imf,\Imc\models \phi
     \text{ and }\Imf,\Imc\models \psi \\
  \Imf,\Imc\models \neg \phi & \text{iff not } \Imf,\Imc\models \phi
    &  \Imf,\Imc\models \K_i \beta & \text{iff } \forall(\Imc,\Jmc) \in \Rmc_i\text{, } \Jmc\models\beta.
 \end{array}
\]
 
An \LK  formula $\varphi$ \emph{entails} an \LK  formula $\psi$, written
  $\varphi\models\psi$,  iff for all pointed \LK interpretations $(\Imf,\Imc)$, 
$\Imf,\Imc\models\varphi$ implies $\Imf,\Imc\models \psi$.
An \LK  formula $\varphi$ is \emph{equivalent} to an \LK  formula $\psi$, written
  $\varphi\equiv\psi$ (we may omit $_\LK$ from $\equiv_\LK$),  iff $\varphi\models\psi$ and $\psi\models\varphi$. 
We use the notion of a set of formulas and the conjunction of its elements 
interchangeably. 
The \emph{size} of a formula or an interpretation $X$, denoted  $| X|$, is the length of the 
string that represents it, where 
concept, role and individual names 
and 
domain elements are considered to be of length $1$.

\subsection{A learning model based on epistemic reasoning}\label{subsec:epistemic-model}

We first adapt the exact learning model 
with membership and equivalence queries to a multi-agent setting. 
We  then introduce the epistemic learning model in a multi-agent 
setting and provide complexity notions for these models.


We   introduce basic notions for the definition of a learning framework 
and the learning problem via queries~\cite{angluinqueries}, adapted   
to a \emph{learning from entailments} 
setting~\cite{DBLP:conf/icml/FrazierP93} with multiple agents.
A \emph{learning (from entailments) framework} $\Fmf$ is a pair $(X, L)$, where $X$ 
is a set of \emph{examples} (also called \emph{domain} or \emph{instance space}), and 
$L$ is a set of  \emph{formulas} of a  description logic \Lmc.
We say that $x\in X$ is a \emph{positive example} for $l \in L$ if $l\models_\Lmc x $ and a
\emph{negative example} for $l$ if $l\not\models_\Lmc x$. 
A \emph{ counterexample} $x$ for $l\in L$ and $h\in L$ is either
a \emph{positive example} for $l$ such that $h\not\models_\Lmc x $
or
a \emph{negative example} for $l$ such that $h\models_\Lmc x $. 
A \emph{multi-agent learning framework} $\Fmf(\agents)$ is a set 
$\{\Fmf_i= (X_i, L_i)\mid i\in \agents\}$ of 
learning frameworks. 

We first provide a formal definition of the exact learning model, based on 
membership and equivalence queries, and then we introduce the 
epistemic learning model, with example and epistemic membership queries. 
Let $\Fmf(\agents)$ be a multi-agent learning framework.  
Each $i\in\agents$ aims at learning  a \emph{target}   
formula $l_j\in L_j$ of a description logic \Lmc 
 of each other agent $j\neq i\in\agents$
by posing them queries. 
 
\begin{definition}[Membership query] For every $i\in\agents$ and every 
$l_i \in L_i$, let  ${\sf MEM}_{\Fmf(\agents),l_i}$ be an oracle that takes as input   $x \in X_i$ and
outputs `yes' if $l_i\models_\Lmc x $ and `no' otherwise. 
A \emph{membership query to agent $i\in\agents$} is a call to  ${\sf MEM}_{\Fmf(\agents),l_i}$.
%
\end{definition}
\begin{definition}[Equivalence query] For every $i\in\agents$ and every 
$l_i \in L_i$, we denote by ${\sf EQ}_{\Fmf(\agents),l_i}$ an oracle 
that takes as input a \emph{hypothesis} formula of a description logic  $h \in L_i$ and
returns `yes' if $h\equiv_\Lmc l_i$ and a 
counterexample for $l_i$ and $h$ otherwise. 
An \emph{equivalence query to agent $i\in\agents$} is a call to  ${\sf EQ}_{\Fmf(\agents),l_i}$.
There is no assumption  about which counterexample 
is returned by ${\sf EQ}_{\Fmf(\agents),l_i}$.
\end{definition}

In this work, we  introduce   \emph{example} queries, 
where an agent $i\in\agents$ can ask an agent $j\in\agents$ to only provide examples which are not logical 
consequences of what they have already communicated. 
Intuitively, if agent $j$ returns $x$ to agent $i$ in a language \Lmc and $x\models_\Lmc y$ 
then agent $i$ knows $y$, in symbols, $\K_i  y$. Since agent $j$ returned this example
to agent $i$, the axiom $\K_i  y$ is part of the logical theory representing 
the knowledge of agent $j$, so agent $j$ acquires knowledge of what is 
known by agent $i$ as they communicate. 
We use example queries 
in combination with an epistemic version of membership queries, 
  called  \emph{\K-membership} queries.  
Given $i\in\agents$, assume  that  $L_i$ is a set of formulas of the logic \Lmc and 
  denote by $L^\K_i$ the set of all formulas in the epistemic extension 
of \Lmc, which, by definition of \LK, includes all \Lmc formulas. 
The target formula  $l_i$ is an element of $L_i$,
however,  the oracles for the example and \K-membership queries 
may add \LK formulas 
to $l_i$. 
We denote by $l^{k+1}_i$ the result of updating $l^{k}_i$
upon receiving the $k$-th query, where  $l^1_i=l_i$.
At all times $X_i$ is a set of examples in \Lmc (not in \LK).

\begin{definition}[\K-membership query] For every $i\in\agents$ and every 
$l^k_i \in L^\K_i$, let  ${\sf MEM}^\K_{\Fmf(\agents),l^k_i}$ be an oracle 
that takes as input   $x \in X_i$ and $j\in\agents$, and, if $l^1_i\models_\Lmc x $, it
outputs `yes' and define $l^{k+1}_i:=l^{k}_i\wedge \K_j x$\footnote{
We may write $l^{k}_i$ for  the conjunction of its elements.}.
 
Otherwise it returns `no' and defines $l^{k+1}_i:=l^{k}_i$. 
The $k$-th \emph{\K-membership query to agent $i\in\agents$} is a call to 
 ${\sf MEM}^\K_{\Fmf(\agents),l^k_i}$.
%
\end{definition}
\begin{definition}[Example query] For every $i\in\agents$ and every 
$l^k_i \in L^\K_i$, let  ${\sf EX}_{\Fmf(\agents),l^k_i}$ be an oracle that 
takes as input some  $j\in\agents$
and outputs  $x \in X_i$ such that  $l^1_i\models_\Lmc x $ 
but $l^k_i\not\models_\LK\K_j x $ 
 if such $x$ exists;
or `you  finished', otherwise. 

Upon returning $x\in X_i$ such that $l^1_i\models_\Lmc x$ 
the oracle ${\sf EX}_{\Fmf(\agents),l^k_i}$  
defines $l^{k+1}_i:=l^{k}_i\wedge \K_j x$. 
The $k$-th \emph{example query to agent $i\in\agents$} 
is a call to   ${\sf EX}_{\Fmf(\agents),l^k_i}$. 
\end{definition}

An \emph{exact learning algorithm} $A_{i}$ 
for $\Fmf_i\in\Fmf(\agents)$ is a deterministic algorithm 
 that takes no input,
 is allowed to make queries to ${\sf MEM}_{\Fmf(\agents),l_i}$ and
${\sf EQ}_{\Fmf(\agents),l_i}$ (without knowing what the target $l_i$ to be learned is),
and eventually halts and outputs some $h\in L_i$ with
$h\equiv_\Lmc l_i$.
An \emph{epistemic learning algorithm} for $\Fmf_i\in\Fmf(\agents)$ is a deterministic  
algorithm that takes no input,
 is allowed to make queries to ${\sf MEM}^\K_{\Fmf(\agents),l^k_i}$  and
 ${\sf EX}_{\Fmf(\agents),l^k_i}$ (without knowing what the target $l^1_i$ to be learned is),
and eventually halts after receiving `you  finished' from ${\sf EX}_{\Fmf(\agents),l^k_i}$.

We say that $\Fmf(\agents)$ is \emph{exactly learnable} 
if there is an exact learning
algorithm $A_{i}$ for each  
$\Fmf_i\in \Fmf(\agents)$ and that $\Fmf(\agents)$ is \emph{polynomial query exactly learnable}
if     each        $\Fmf_i\in\Fmf(\agents)$ is exactly learnable by an algorithm $A_{i}$ 
such that at every step
the sum of the sizes of the inputs to 
queries made by $A_{i}$ up to that step is bounded by a polynomial
$p(|l_i|,|x|)$, where $l_i$ is the target and $x \in X_i$ is the largest
example seen so far by $A_{i}$.  
$\Fmf(\agents)$ is
\emph{ polynomial time exactly learnable} if each $\Fmf_i\in \Fmf(\agents)$   is exactly learnable by an
 algorithm $A_{i}$ 
such that 
at every step (we count each call to an oracle
as one step of computation) of computation the time used by $A_{i}$ up to
that step is bounded by a polynomial $p(|l_i|,|x|)$, where $l_i\in L_i$
is the target and $x \in X$ is the largest counterexample seen so far.
We  may also say that $\Fmf(\agents)$ is learnable in $O(|l_i|,|x|)$ many steps, 
 following the same notion of polynomial time learnability, except that 
the number of steps is bounded by $O(|l_i|,|x|)$.

We say that $\Fmf(\agents)$ is \emph{ epistemically learnable } 
if there is an epistemic learning
algorithm for each $\Fmf_i\in \Fmf(\agents)$.   
Polynomial query/time epistemic learnability is defined analogously, 
with $p(|l^1_i|,|x|)$ defined in terms of $|l^1_i|$ and $|x|$. 
Clearly, if a learning framework $\Fmf(\agents)$   is polynomial time exactly/epistemically learnable
then it is also polynomial query exactly/epistemically learnable.

\section{Epistemic and exact polynomial learnability}\label{sec:epiVSexac} 
\newcommand{\setone}[1]{\ensuremath{s^{\Kmc}_{{#1}}}\xspace} 
\newcommand{\settwo}[1]{\ensuremath{s^{\Lmc}_{{#1}}}\xspace} 

In this section we confront polynomial query and polynomial time learnability 
in the exact and epistemic learning models.
We start by considering the case where the learner is only allowed 
to pose one type of query. Clearly,  polynomial (query/time) exact learnability 
 with only membership queries
coincides with  polynomial   epistemic learnability with only \K-membership queries. 
We  now analyse polynomial learnability with equivalence queries only and 
example queries only. Our first result is that polynomial (query/time) learnability 
in the epistemic learning model implies polynomial learnability in exact learning model. 

 \begin{theorem}\label{thm:onedirection}
If a multi-agent learning framework   is  polynomial query (resp. time) epistemically learnable 
with only example queries then it is   polynomial query (resp. time) exactly learnable with only 
equivalence queries. 
\end{theorem}
\begin{proof}
 Assume $\Fmf(\agents)$ is 
polynomial query epistemically learnable with only example queries
(the case of polynomial time epistemic learnability with only example queries can be 
similarly proved). 
For each $\Fmf_i\in\Fmf(\agents)$ there is an epistemic learning algorithm 
$A_{i}$  for $\Fmf_i$ with polynomial query complexity 
which only asks example queries.
To construct an exact learning algorithm $A'_{i}$ for $\Fmf_i$ which only 
asks equivalence queries using
$A_{i}$, we define   auxiliary sets $\setone{i}(k)$ and $\settwo{i}(k)$ which 
will keep  the information returned by ${\sf EQ}_{\Fmf(\agents),l_i}$ 
 up to the $k$-th query posed by 
a fixed but arbitrary agent in $\agents\setminus\{i\}$ and agent $i$.
We define $\setone{i}(1)=\emptyset$ and $\settwo{i}(1)=\emptyset$.
\begin{itemize}
\item Whenever  $A_{i}$ poses an example query to  agent $i\in\agents$  
(assume it is the $k$-th query),
  $A'_{i}$ calls the oracle ${\sf EQ}_{\Fmf(\agents),l_i}$ 
with $\settwo{i}(k)$ as input. The oracle either returns `yes' if 
$\settwo{i}$ is equivalent to $l_i$ or it returns some 
  counterexample for $l_i$ and $\settwo{i}(k)$ 
  (we may write $\settwo{i}(k)$ to denote $\bigwedge_{\beta\in\settwo{i}(k)}\beta$).  
Then $A'_{i}$ adds $\K_j x$ to $\setone{i}(k)$ and $x$ to $\settwo{i}(k)$.  
\end{itemize}  
If ${\sf EQ}_{\Fmf(\agents),l_i}$ returns `yes' then 
 algorithm $A'_{i}$ converts it into `you finished', as expected by 
 algorithm $A_{i}$. 
We now argue that, for all $x\in X_i$ and all $k\geq 0$ such that $l_i\models_\Lmc x$,
we have that $x$ is a (positive) counterexample  for  $l_i$ and $\settwo{i}(k)$ 
iff  
$l_i\wedge \setone{i}(k)\not\models_\LK\K_j x $. 
By definition of $\setone{i}(k)$ and   $\settwo{i}(k)$ and since $l_i$ 
does not contain \LK axioms,  for all $x\in X_i$ and all $k\geq 0$, we have that
$l_i\wedge \setone{i}(k)\models_\LK  \K_j x$ 
iff $\settwo{i}(k) \models_\Lmc x$.
By definition and construction of $\settwo{i}(k)$,
it follows that $ l_i\models_\Lmc\settwo{i}(k)$. 
So $\settwo{i}(k)\not\models_\Lmc x  $ iff $l_i\wedge \setone{i}(k)\not\models_\LK 
\K_j x  $. 
Hence ${\sf EQ}_{\Fmf(\agents),l_i}$ can simulate ${\sf EX}_{\Fmf(\agents),l^k_i}$, 
where $k$ represents the number of calls to ${\sf EX}_{\Fmf(\agents),l^k_i}$ 
posed so far by $A_i$.  
By definition of $A_{i}$, at every step,
the sum of the sizes of the inputs to 
queries made by $A_{i}$ up to that step is bounded by a polynomial
$p(|l_i|,|x|)$, where $l_i$ is the target and $x \in X_i$ is the largest 
counterexample seen so far by $A_{i}$.
%
Then, for all $k\geq 0$, we have that $|\settwo{i}(k)|\leq|\setone{i}(k)|\leq p(|l_i|,|x|)$. 
Since all responses to queries are as required 
by $A_{i}$, if $A_{i}$ halts 
after polynomially many polynomial size 
queries, the same happens with $A'_{i}$, which 
returns a hypothesis $\settwo{i}(k)$ equivalent to the target $l_i$, for 
some $k\leq p(|l_i|,|x|)$. 
\end{proof}

The converse of Theorem~\ref{thm:onedirection} does not hold, 
as we show in the next theorem. 

\begin{theorem}\label{thm:harder}
There is a multi-agent learning framework $\Fmf(\agents)$ such that 
$\Fmf(\agents)$  is  polynomial time exactly learnable with only 
equivalence queries but not  polynomial query (so, not polynomial time) 
epistemically learnable 
with only example queries. 
\end{theorem}
\begin{proof}
Consider the learning framework $\Fmf=(X, L)$ where 
$X$ is the set of propositional formulas over the variables 
${\sf Prop}=\{q,p,p^0_1,\ldots,p^0_n,p^1_1,\ldots,p^1_n\}$ and $L=\{\varphi\mid 
\varphi \in X, \varphi \equiv (p\rightarrow q)\}$ (where $\equiv$ denotes 
logical equivalence in propositional logic). So 
the target can only be a formula equivalent to $p\rightarrow q$. 
Now   let $\Fmf(\agents)$ be the set 
$\{\Fmf_i=(X, L)\mid i \in \agents\}$, with all learning frameworks 
are equal to $\Fmf$ (this does not mean that the target is the same 
for all agents but that they are taken from the same set $L$).
If $L$ is a language 
which only contains propositional formulas equivalent to $p\rightarrow q$, 
an exact learning algorithm can learn the target with 
only one equivalence query, passing  the hypothesis $\{p\rightarrow q\}$ as input.
However, ${\sf EX}_{\Fmf(\agents),\{p\rightarrow q\}}$ can return 
any of the exponentially many examples of the form $p\wedge
 (p^{\ell_1}_1\wedge \ldots\wedge p^{\ell_n}_n)\rightarrow q$,
with $\ell_j\in\{0,1\}$ and  $j\in \{1,\ldots,n\}$.  
The example oracle can always provide an example which does not 
follow from its knowledge of what is known by the learner  by taking a fresh binary sequence.  
Thus, there is no epistemic algorithm which can learn the target 
with polynomially many queries. 
\end{proof}

Interestingly, if we consider both types of queries then  polynomial exact learnability  
  \emph{coincides} with  polynomial epistemic learnability.
\begin{restatable}{theorem}{TheoremTransfer}\label{thm:transfer}
Let $\Fmf(\agents)$ be a multi-agent learning framework.  
 $\Fmf(\agents)$ is  polynomial query (resp.   time) 
  exactly learnable if, and only if, 
$\Fmf(\agents)$ is  polynomial query (resp.   time)  epistemically  learnable. 
\end{restatable}
\begin{proof} 
($\Rightarrow$) In our proof we use  polynomial query exact learnability, the argument 
for  polynomial time exact learnability is analogous.
Assume $\Fmf(\agents)$ is polynomial query 
exactly learnable. 
Then, for each $\Fmf_i\in\Fmf(\agents)$ there is an exact learning algorithm 
$A_{i}$  for $\Fmf_i$. 
We construct an epistemic learning algorithm $A'_{i}$ for $\Fmf_i$ using
$A_{i}$ as follows. Recall that  
we write $l^k_i$ to denote the target $l^k_i$ after the $k$-th query (Section~\ref{subsec:epistemic-model}). 
\begin{itemize}
\item Whenever $A_{i}$ poses a membership query to  agent $i\in\agents$ 
with $x\in X_i$ as input,
$A'_i$ calls   ${\sf MEM}^\K_{\Fmf(\agents),l^k_i}$ 
with $x$ as input, where $k$ represents the number of queries posed 
so far by $A'_{i}$.
\item Whenever $A_{i}$ poses an equivalence query to  agent $i\in\agents$ 
with a hypothesis $h$ as input, we have that, for each $x\in h$, $A'_i$ calls  
 ${\sf MEM}^\K_{\Fmf(\agents),l^k_i}$ with $x$ as input (and $k$ is incremented).
Then, the algorithm calls the oracle ${\sf EX}_{\Fmf(\agents),l^k_i}$.
\end{itemize}  
  ${\sf MEM}^\K_{\Fmf(\agents),l_i}$  behaves as it 
is required by algorithm $A_{i}$ to learn $\Fmf_i$. 
We show that 
whenever 
${\sf EX}_{\Fmf(\agents),l^k_i}$ returns some $x\in X_i$ we have that 
$x$ is a counterexample for $l^1_i$ and $h$, where 
$h$ is the input  of the 
equivalence query 
posed  by $A_{i}$.
By definition of $A_{i}$, at every step,
the sum of the sizes of the inputs to 
queries made by $A_{i}$ up to that step is bounded by a polynomial
$p(|l^1_i|,|x|)$, where $l^1_i$ is the target and $x \in X_i$ is the largest 
example seen so far by $A_{i}$.
Let $h^\ell$ denote the input to the $\ell$-th
equivalence query posed by $A_{i}$.
For all $\ell > 0$, we have that $|h^\ell|\leq p(|l^1_i|,|x|)$. 
The fact that $x$ is a counterexample for $l^1_i$ and $h^\ell$ 
follows from the definition of $A'_i$, which poses
 membership queries for each $x\in h^\ell$, 
 ensuring that  $l^k_i$ is  updated with $\K_j x$ after each query. 
Hence  ${\sf EX}_{\Fmf(\agents),l^k_i}$ 
returns counterexamples for $l^1_i$ and $h^\ell$  (if they exist),
as ${\sf EQ}_{\Fmf(\agents),l^1_i}$.  
Since $A_{i}$ poses only polynomially many queries, $\ell$ is bounded 
by  $p(|l^1_i|,|x|)$.
So the sum of the sizes of the inputs to  
queries made by the epistemic learning algorithm $A'_{i}$ simulating $A_{i}$  
  is quadratic in $p(|l^1_i|,|x|)$. 
All in all, since all responses to queries are as required 
by $A_{i}$, if $A_{i}$ halts and outputs some $h\in L_i$ with
$h\equiv_\Lmc l^1_i$ (with $h$ the input to the last equivalence query)  
after polynomially many polynomial size 
queries, we have that ${\sf EX}_{\Fmf(\agents),l^1_i}$  is forced to return 
`you finished'
and so $A'_{i}$ also halts after polynomially many polynomial size 
queries. 
The ($\Leftarrow$)  direction is similar 
to the proof of Theorem~\ref{thm:onedirection}, except that we 
now also have (\K-)membership queries. 
\end{proof}


%% file: epistemicEL.tex
\section{The epistemic \EL description logic}\label{sec:reasoning}


To instantiate the multi-agent epistemic learning problem to the \EL 
case, in Section~\ref{sec:learningEL}, we define and study in this section the 
epistemic extension of \EL, called \ELK.
We present \EL~\cite{dlhandbook} in Section~\ref{sec:el}.
\ELK is   the instantiation of \LK presented in Section~\ref{subsec:extension} with the logic \EL.
We establish the complexity of the satisfiability problem for \ELK in Section~\ref{sec:reasoning-elk} and of one of its fragments in Section~\ref{sec:reasoning-conj-elk}.

We showed in Section~\ref{sec:epiVSexac}   that example queries give strictly less power to the learner than equivalence queries.
We also argued, quite informally so far, that example queries are less demanding on the oracle 
than equivalence queries.
Instead of deciding whether two ontologies are equivalent, and then providing a counterexample 
when it is not the case, the oracle only needs to reason about what they know about the knowledge of the learner.
Yet, we did not say anything about the actual complexity of the epistemic reasoning involved in example queries.
If reasoning about the knowledge of the learner is harder than 
evaluating the equivalence of two ontologies, then the advantage of example queries 
for the oracle would be moot.
We show that indeed the epistemic reasoning that 
the oracle needs to perform is in 
\PTime (Theorem~\ref{th:conj-elk}).
So, the oracle's benefit from example queries 
over equivalence queries is a net benefit.

\subsection{\EL: syntax, semantics, and complexity}
\label{sec:el}

\EL  concepts  $C,D$ are expressions of the form:
$
C,D ::= \top\mid A \mid \exists r.C\mid    C\sqcap D 
$
where $A \in \NC $ and $r \in \NR $.
An \emph{inclusion} is an expression of the form $C\sqsubseteq D$
where $C,D$ are \EL  concept expressions; and 
an \emph{assertion} is of the form $ A(a)$ or $r(a,b)$
with $a,b\in\NI$, $A\in\NC$, and $r\in \NR$. 
An \emph{\EL axiom} is an inclusion or an assertion.  
An \emph{\EL formula}\footnote{Typically an \EL \emph{ontology} is a set of \EL axioms~\cite{dlhandbook}, and can also be seen as a conjunction of positive \EL axioms.
Here we also consider \EL \emph{formulas}, where we allow negations and conjunctions 
over the axioms.} is an expression of the form
$
\alpha ::=   \elaxiom \mid \neg \alpha\mid \alpha\wedge\alpha   
$ where $\elaxiom$ is an \EL axiom. 
An \EL literal is an \EL axiom or its negation. 
The semantics of \EL is given by \Lmc interpretations $\Imc=(\Delta^\Imc,\cdot^\Imc)$ as defined 
in Section~\ref{subsec:extension}, considering  $\Lmc=\EL$.
We extend the mapping  $\cdot^\Imc$ for \EL complex concept expressions as follows:
\begin{gather*}
\top^{\Imc} := \Delta^\Imc, \qquad
(C \sqcap D)^{\Imc} := C^{\Imc} \cap D^{\Imc}, \\
(\exists r .C)^{\Imc} := \{d \in \Delta^\Imc \mid \exists  e \in C^{\Imc}: (d,e) \in r^{\Imc} \}.
\end{gather*}
We now define the entailment relation $\models_{\EL}$ for \EL formulas. 
Given an \EL interpretation \Imc and an \EL axiom (which can be an inclusion or an assertion, as above) we define:
$\Imc\models_{\EL} C\sqsubseteq D  \text{ iff }  C^\Imc\subseteq D^{\Imc}$;
$\Imc\models_{\EL} A(a)  \text{ iff }  a^\Imc\in A^{\Imc}$; and 
$\Imc\models_{\EL} r(a,b)   \text{ iff } (a^\Imc,b^\Imc)\in r^\Imc $. 
We   
  inductively extend the relation $\models_{\EL}$ to   \EL formulas   as in 
Section~\ref{subsec:extension}:  
$\Imc\models_{\EL} \varphi\wedge\psi  \text{ iff }  \Imc\models_{\EL} 
\varphi$ and $\Imc\models_{\EL} \psi$; and 
$\Imc\models_{\EL} \neg\varphi   \text{ iff not }  \Imc\models_{\EL} 
\varphi$.
In our proofs, we use the following result.
\begin{restatable}{lemma}{LemmaComplConjElLit}\label{lem:compl-conj-el-lit}   
Satisfiability of a conjunction of \EL literals is \PTime-complete~\cite{DBLP:conf/ijcai/BorgwardtT15}.
\end{restatable}
We establish in Section~\ref{sec:reasoning-elk} that 
reasoning about \ELK formulas is \NP-complete, 
just like reasoning about \EL formulas.
We note that $\mathcal{EL}(\mathcal{K})$  formulas allow arbitrary Boolean 
combinations of $\mathcal{EL}(\mathcal{K})$ axioms, 
hence the contrast with the \PTime complexity of entailment from an \EL ontology~\cite{BBL-EL}. In Section~\ref{sec:reasoning-conj-elk} we show that reasoning about \ELK restricted to conjunctions of literals is in \PTime.

\subsection{Reasoning in \ELK}
\label{sec:reasoning-elk}

Here we study  the complexity of the satisfiability problem 
in \ELK. 
Our combination of epistemic logic and description logic 
is orthogonal to the work by De Giacomo et.~al~\cite{DBLP:conf/kr/GiacomoINR96}:
while our epistemic operators are over \EL formulas, 
the epistemic operators of the mentioned work are over 
concepts and roles. For instance, there, $\K {\sf FrenchChef}$ denotes 
the concept of \emph{known} French chefs. Here,    \ELK contains formulas such as $\K_i({\sf FrenchChef(Soyer)}) 
\land \lnot\K_i\K_j({\sf Crepe} \sqsubseteq \exists {\sf contains. Egg})$ 
indicating that agent $i$ knows that Soyer is a French chef, but $i$ does 
not know that $j$ knows that crepes contain egg.

From the definition of the language of \LK in Section~\ref{subsec:extension}, remember that the language of \ELK does not admit alternating modalities; E.g., $\K_i\lnot \K_j A(a)$ is not a formula of \ELK. It is rather easy to see that if there were no such syntactic restrictions, the satisfiability problem would turn out to be \PSpace-complete. (We could reduce satisfiability and adapt the tableaus method of propositional $S5_n$~\cite{HALPERN1992319}.)
Instead, we establish that satisfiability in \ELK is \NP-complete.

The lower bound follows from \NP-hardness for propositional logic. 
The following lemma is central for showing membership in \NP. It is a
consequence of the fact that \EL and propositional logic have 
the polynomial size model property and that in \ELK   
the satisfiability test can be separated into two independent tests: 
one for 
the DL dimension and one for the epistemic dimension
(see~\cite{Baader:2012:LOD:2287718.2287721,DBLP:conf/ijcai/BorgwardtT15}).
\begin{restatable}{lemma}{LemmaPolyModel}\label{lem:poly-size}
\ELK enjoys the polynomial size model property. 
\end{restatable}

Since \ELK formulas can be translated into 
first-order logic, for a fixed \ELK formula $\varphi$, 
  checking whether a polynomial size interpretation is a model of $\varphi$ 
can be performed in \NL. 
Thus, membership in \NP is by the fact that, 
by Lemma~\ref{lem:poly-size}, one can guess a polynomial size 
model (if one exists) and check that it is a model 
in $\NL\subseteq\PTime$.

\begin{restatable}{theorem}{TheoremELKCompl}\label{th:elk}
  Satisfiability in \ELK is \NP-complete.
\end{restatable}

%

%
%

\subsection{Reasoning in conjunctive \ELK}
\label{sec:reasoning-conj-elk}
\newcommand{\ELPP}[0]{$\EL^{++}$\xspace}
\newcommand{\KK}{\ensuremath{{\mathbb K}}\xspace}
\newcommand{\elaxiomlit}{\ensuremath{{\sf l}}\xspace}

We conclude this section considering the satisfiability problem 
for \emph{conjunctive \ELK}, defined 
as the fragment of \ELK which only allows negations 
in front of \EL axioms or 
in front of \ELK 
axioms of the form $\KK  \alpha$, with $\alpha$ a conjunction of \EL literals 
and $\KK $  a non-empty sequence of epistemic 
operators. 
%
Formally, conjunctive \ELK formulas $\varphi$ are expressions of the   form: 
$
 \varphi ::=  \alpha \mid \beta \mid \lnot \beta \mid \varphi\wedge \varphi
$
 with
 $\beta ::= \K_i \alpha \mid  \K_i \beta$,
 and 
$\alpha ::=   \elaxiom \mid \neg \elaxiom\mid \alpha\wedge\alpha$,
where $\elaxiom$ is an \EL axiom. 

\begin{algorithm}[t]
\caption{$SAT(\varphi)$, deciding the satisfiability of conjunctive \ELK formulas\label{A:conjunctive-elk-sat}}
\begin{algorithmic}[1]
\Require A conjunctive \ELK formula $\varphi$
\Ensure TRUE if $\varphi$ is satisfiable, and FALSE otherwise
\If {$\omega_0 \land \bigwedge \{ \omega \mid \KK_\sigma \omega \in \varphi^\flat \}$ is not \EL satisfiable}
\State \Return{FALSE}
\EndIf
\For {$\lnot \KK_{\sigma}\omega \in \varphi^\flat$}
\State \label{lab:algo-start-check} $\psi = \top \land \bigwedge \{ \omega' \mid \KK_{\sigma'} 
\omega' \in \varphi^\flat, \text{ and } \sigma \text{ is a subword of } \sigma' \}$
\State $MS = \{\psi \land \lnot \beta \mid \beta \text{ is an \EL literal in }  \omega \}$
\If{all conjunctions of \EL literals in $MS$ are not \EL satisfiable} 
\State \Return {FALSE}
\EndIf \label{lab:algo-end-check}
\EndFor
\State \Return {TRUE}
\end{algorithmic}
\end{algorithm}
To establish the complexity of reasoning in conjunctive \ELK, we use the following notation.
For every non-empty sequence $\sigma = a_1 \ldots a_k \in \agents^+$ of agents,
 we associate a sequence    $\KK_\sigma = \K_{a_1} \ldots \K_{a_k}$ 
 of epistemic operators. 
We write $\beta \in \psi$ if $\beta$ is a conjunct occurring in $\psi$.
%
We say that $\sigma'\in \agents^+$ is a \emph{subword} of $\sigma\in \agents^+$ 
when $\sigma'$ is the result of removing zero or more elements from $\sigma$ (at any position 
of the sequence).
Given a conjunctive $\ELK$ formula 
\[
\varphi=\omega_0 \land
\KK_{\sigma_1} \omega_{\sigma_1} \land \ldots \land \KK_{\sigma_n} \omega_{\sigma_n} \land
\lnot \KK_{\sigma_{n+1}} \omega_{\sigma_{n+1}} \land \ldots \land \lnot \KK_{\sigma_m} \omega_{\sigma_m} \enspace  
\]
where $\sigma_i \in \agents^+$, for every $1 \leq i \leq m$,  
and each $\omega_i$, with $0 \leq i \leq m$, is a conjunction of \EL literals,
we denote by $\varphi^\flat$ the formula resulting from 
exhaustively substituting in $\varphi$ every adjacent repetitions 
$a \ldots a$ of an agent $a$ occurring in $\sigma_i$, $1 \leq i \leq m$, with $a$. (E.g., $a_1a_2a_2a_3a_2$ becomes $a_1a_2a_3a_2$.)

The following proposition is central to the correctness of Algorithm~\ref{A:conjunctive-elk-sat}.
\begin{restatable}{proposition}{PropCorrConjEL}\label{prop:corr-conj-EL}
%
  A conjunctive \ELK formula $\varphi$ is unsatisfiable iff at least one of the following properties holds:
  \begin{enumerate}
  \item $\omega_0 \land \bigwedge \{ \omega \mid \KK_{\sigma} \omega \in \varphi^\flat \}$ is not \EL satisfiable;
  \item there is $\lnot \KK_{\sigma}\omega \in \varphi^\flat$ such 
  that $\lnot\omega \land \bigwedge \{ \omega' \mid \KK_{\sigma'}\omega' \in 
  \varphi^\flat, \text{ and } \sigma\linebreak  \text{ is a subword of } \sigma' \}$ is not \EL satisfiable.
  \end{enumerate}
\end{restatable}

Proposition~\ref{prop:corr-conj-EL} suggests that the satisfiability of conjunctive \ELK formulas can be reduced to checking the satisfiability of a few conjunctions of \EL literals.
We are finally ready to prove the complexity of deciding whether a conjunctive \ELK formula is satisfiable.

\begin{restatable}{theorem}{TheoremConjunction}\label{th:conj-elk}
  Satisfiability in conjunctive \ELK is \PTime-complete.
\end{restatable}
\begin{proof}
  %
  Consider Algorithm~\ref{A:conjunctive-elk-sat}. The conjunctive \ELK formula $\varphi$ is 
  satisfiable iff $SAT(\varphi)$ returns TRUE. The correctness of the algorithm 
  follows immediately from Prop.~\ref{prop:corr-conj-EL}. It suffices to 
  observe that Lines~\ref{lab:algo-start-check}--\ref{lab:algo-end-check} check the unsatisfiability of an \EL formula $\lnot\omega \land \psi$ where $\omega$ and $\psi$ are two of conjunctions of \EL literals ($\lnot\omega \land \psi$ is \emph{not} a conjunction of \EL literals, unless $\omega$ contains only one literal) by checking the unsatisfiability of as many conjunctions of \EL literals $\lnot \beta \land \psi$ as there are literals $\beta$ in $\omega$.
A simple analysis shows that the algorithm runs in time polynomial 
in the size of $\varphi$, with a polynomial number of calls to a procedure 
for checking the unsatisfiability of conjunctions of \EL literals. By Lemma~\ref{lem:compl-conj-el-lit}, 
  each of these checks can be done in polynomial time. Membership in \PTime follows.
\end{proof}


%% file: learningEL.tex
\section{Learning  \EL  with epistemic reasoning } \label{sec:learningEL}
It is known that \EL ontologies are not polynomial query exactly learnable, while 
the   fragments of \EL restricting 
one of the sides of inclusions to be a concept name,
namely $\EL_{\sf lhs}$ and $\EL_{\sf rhs}$,  
 are exactly learnable in polynomial time~\cite{KLOW18}.
 In this section, 
 we transfer results known for \EL and its fragments 
 to our learning model. 
%
%
Our results  are for learning frameworks 
where the learning language is the same for all agents. 
That is, we deal with
the special case of a 
multi-agent learning framework $\Fmf(\agents)=\{\Fmf_i= (X_i, L_i)\mid i\in \agents\}$
   where all formulas in all $L_i$ are 
 from a DL $\Lmc$,    denoted $\Fmf(\Lmc,\agents)$.
Theorem~\ref{thm:expel} is a consequence of Theorem~\ref{thm:transfer} 
and   complexity results for \EL and its fragments in 
the exact learning model~\cite{KLOW18}. 

\begin{theorem}\label{thm:expel}
The learning framework $\Fmf(\EL,\agents)$ is not 
polynomial query epistemically learnable.
The learning frameworks $\Fmf(\EL_{\sf lhs},\agents)$ and  $\Fmf(\EL_{\sf rhs},\agents)$
are polynomial  time epistemically learnable. 
\end{theorem} 
 
The hardness result for \EL holds even for the fragment of \EL ontologies 
defined as the union of $\EL_{\sf lhs}$ and $\EL_{\sf rhs}$, that is,
in a named form where at least one of the sides of concept inclusions 
is a concept name,  
which we call   $\EL_{{\sf lhs}, {\sf rhs}}$.  
An implementation of a learning algorithm for \EL ontologies in this named form 
was presented by Duarte et.~al~\cite{DBLP:conf/dlog/DuarteKO18,DBLP:conf/kr/DuarteKO18}. 
The algorithm is exponential in the size of the vocabulary $\Sigma_\Omc$ of the ontology  $\Omc$ 
(which is the set 
of concept/role and individual names occurring in \Omc)
and the largest concept expression $C_\Omc$\footnote{`The largest' concept expression (and, later, counterexample) 
refers 
to the maximum of the sizes of counterexamples/concept expressions.}, but it is not exponential in the size of the whole 
ontology.

\newcommand{\hai}{\hskip\algorithmicindent}
\begin{algorithm}[t]
\begin{algorithmic}[1]
\Require  An \EL terminology \Omc given to the oracle; $\Sigma_\Omc$ given to the learner
\Ensure An \EL terminology \Hmc computed by the learner such that $\Omc\equiv_\EL\Hmc$
\State $\mathcal{H}:=\{{\sf a}\mid {\sf MEM}^\K_{\Fmf,\Omc^k}({\sf a}){=}\text{`yes'}, \ {\sf a} \text{ is 
a $\Sigma_\Omc$-assertion or } {\sf a}=A\sqsubseteq B, \ A,B\in\Sigma_\Omc\}$\label{ln:first}
\While{$  {\sf EX}_{\Fmf ,\Omc^k}\neq$   `you finished'}
\State Let $C \sqsubseteq D$ be the returned positive 
 example for $\mathcal{O}$  
\State Compute, with ${\sf MEM}^\K_{\Fmf ,\Omc^k}$, $C' \sqsubseteq D'$ such that 
 $C'$ or $D'$   in $\Sigma_\Omc\cap\NC$  
 \label{line:ce}  
\If {$C' \in \Sigma_\Omc\cap\NC$} 
\State Compute with ${\sf MEM}^\K_{\Fmf,\Omc^k}$ a right \Omc-essential   ${\sf a}$ from 
$C' \sqsubseteq D'\sqcap \bigsqcap\limits_{C'\sqsubseteq F'\in \Hmc}F'$ 
\Else
\State Compute with ${\sf MEM}^\K_{\Fmf,\Omc^k}$ a left \Omc-essential   ${\sf a}$ from $C' \sqsubseteq D'$ 
\EndIf
\State Add ${\sf a}$ to $\mathcal{H}$
\EndWhile 
\State \Return {\Hmc}
\end{algorithmic}
\caption{Adaptation of the learning algorithm for $\EL_{\sf lhs,rhs}$~\cite{DBLP:conf/kr/DuarteKO18}
\label{alg:learner}}
\end{algorithm}

Theorem~\ref{thm:transfer} is not directly applicable in this case, 
however, we observe that if the exact learning algorithm uses the epistemic 
learning model, then 
the outcome of each example query
will   be 
a counterexample, and so, the complexity result obtained 
with that algorithm is transferable to the epistemic setting.
To see this, consider Algorithm~\ref{alg:learner}, which 
is an adaptation of the  exact learning algorithm for $\EL_{\sf lhs,rhs}$
 \cite{DBLP:conf/dlog/DuarteKO18,DBLP:conf/kr/DuarteKO18}. 
 Assume \Fmf in Algorithm~\ref{alg:learner} is $\Fmf(\EL_{{\sf lhs}, {\sf rhs}},\agents)$.
The number of $\Sigma_\Omc$-assertions (defined as 
 assertions with only symbols from $\Sigma_\Omc$)
is polynomial 
in the size of \Omc, so, in Line~\ref{ln:first}, Algorithm~\ref{alg:learner} 
identifies those that occur in \Omc using \K-membership queries. 
It follows that all examples returned by the oracle in 
the `while' loop are concept inclusions.  
In each iteration of the `while' loop, the algorithm  uses the examples
returned by the \EXoracle oracle  to 
compute what is called `left \Omc-essential' and `right \Omc-essential'
concept inclusions using \K-membership queries, 
and then updates the hypothesis with such inclusions.
We do not go into details of the algorithm, which is fully presented in the mentioned reference,
but point out that it 
only adds to its hypothesis 
concept inclusions that follow from the target ontology \Omc.

Since we use 
\K-membership queries,   the oracle is aware of the knowledge 
obtained by the learner in this way and does not return examples which 
follow from such entailments. With an inductive argument 
on the number of iterations of the main loop of the algorithm
one can show  that, 
at each iteration, if the learner asks for an example query 
instead of an equivalence query, the outcome will indeed be 
a counterexample for \Omc and \Hmc. So the number of membership and equivalence queries is the same 
as the number of \K-membership and example queries. 
Moreover, the hypothesis \Hmc computed by Algorithm~\ref{alg:learner}  is equivalent to the target \Omc 
(where $\Omc=\Omc^1$, so without epistemic axioms).  
Our next theorem formalises the fact that   
the number of queries performed by the exact learning algorithm has the same bound  
  in the epistemic learning framework. 

\begin{theorem}
$\Fmf(\EL_{{\sf lhs}, {\sf rhs}},\agents)$ 
is epistemically learnable in $O(|\Sigma_\Omc|^{\sharp_\Omc}\cdot(|C\sqsubseteq D|)^2)$ many steps,
where $\sharp_\Omc =2\cdot |C_\Omc| \cdot |\Sigma_\Omc| + 2$,
$C_\Omc$ is the 
largest concept expression in \Omc and $C\sqsubseteq D$ is 
the largest 
 counterexample given by the oracle. 
\end{theorem}

%% file: conclusion.tex
\section{Discussion}
 
We introduced the epistemic learning model and investigated polynomial 
learnability in our model, establishing that it coincides with polynomial 
learnability in the exact learning model, 
and as a consequence, we can also transfer results in our model to 
  the PAC learning model extended with membership queries.  
When the learner is only allowed to pose example queries, we showed 
that polynomial learnability in our model in strictly harder than 
in the exact learning model with only equivalence queries. 
This suggests that example queries are less demanding for the oracle than 
equivalence queries. We   showed that, in the \EL case, 
the epistemic reasoning that 
the oracle needs to perform features \PTime complexity. 
Our results complement previous research
on polynomial learnability in the exact and PAC learning 
models~\cite{DBLP:journals/jcss/BshoutyJT05}, where the authors analyse models 
between the exact and PAC learning models, 
in a learning from interpretations setting. 
As future work, 
we plan to investigate whether 
the implementation  for \ELrl~\cite{DBLP:conf/kr/DuarteKO18}
could benefit from our approach, where the 
oracle keeps track of the knowledge passed to the learner, 
instead of processing the hypothesis at each iteration. 
%





%% file: appendix.tex
\section{Proofs for Section~\ref{sec:learning-epistemic}}

\TheoremTransfer*
\begin{proof}
For completeness of our results we show in full detail the ($\Leftarrow$) direction 
of this theorem. 
Now, assume $\Fmf(\agents)$ is 
polynomial query epistemically learnable 
(we skip the argument 
for polynomial time epistemic learnability as it is similar). 
For each $\Fmf_i\in\Fmf(\agents)$ there is an epistemic learning algorithm 
$A_{i}$  for $\Fmf_i$ with polynomial query complexity.
To construct an exact learning algorithm $A'_{i}$ for $\Fmf_i$   using
$A_{i}$, we define   auxiliary sets $\setone{i}(k)$ and $\settwo{i}(k)$ which 
will keep  the information returned by the oracle  
 up to the $k$-th query posed by  
a fixed but arbitrary agent in $\agents\setminus\{i\}$ and agent $i$.
We define $\setone{i}(1)=\emptyset$ and $\settwo{i}(1)=\emptyset$.
\begin{itemize}
\item Whenever $A_{i}$ poses the $k$-th \K-membership query to  agent $i\in\agents$ 
with $x\in X_i$ as input,
we call the   oracle ${\sf MEM}_{\Fmf(\agents),l_i}$ 
with $x$ as input, and if $l_i\models_\Lmc x$,  we add $\K_j x$ to the set 
$\setone{i}(k)$ and $x$ to the set $\settwo{i}(k)$. 
\item Whenever $A_{i}$ poses the $k$-th example query to  agent $i\in\agents$
we call the oracle ${\sf EQ}_{\Fmf(\agents),l_i}$ 
with $\settwo{i}(k)$ as input. The oracle either returns `yes' if 
$\settwo{i}$ is equivalent to $l_i$ or it returns some 
  counterexample $x$ for $l_i$ and $\settwo{i}(k)$. 
We   add $\K_j x$ to $\setone{i}(k)$ and $x$ to $\settwo{i}(k)$ .  
\end{itemize}  
  ${\sf MEM}_{\Fmf(\agents),l_i}$  behaves as it 
is required by algorithm $A_{i}$ to learn $\Fmf_i$. 
If ${\sf EQ}_{\Fmf(\agents),l_i}$ returns `yes' then 
 algorithm $A'_{i}$ converts it into `you finished', as expected by 
 algorithm $A_{i}$. 
We now show that, for all $x\in X_i$ such that $l_i\models_\Lmc x$,
$x$ is a (positive) counterexample  for  $l_i$ and $\settwo{i}(k)$ 
iff  
$l_i\wedge \setone{i}(k)\not\models_\LK\K_j x $. 
By definition of $\setone{i}(k)$ and   $\settwo{i}(k)$ and since $l_i$ 
does not contain \LK axioms,  for all $x\in X_i$, we have that
$l_i\wedge \setone{i}(k)\models_\LK  \K_j x$ 
iff $\settwo{i}(k) \models_\Lmc x$.
By definition and construction of $\settwo{i}(k)$,
it follows that $ l_i\models_\Lmc\settwo{i}(k)$. 
So $\settwo{i}(k)\not\models_\Lmc x  $ iff $l^k_i\wedge \setone{i}(k)\not\models_\LK 
\K_j x  $. 
Hence ${\sf EQ}_{\Fmf(\agents),l_i}$ can simulate ${\sf EX}_{\Fmf(\agents),l^k_i}$.  
%
By definition of $A_{i}$, at every step,
the sum of the sizes of the inputs to 
queries made by $A_{i}$ up to that step is bounded by a polynomial
$p(|l_i|,|x|)$, where $l_i$ is the target and $x \in X_i$ is the largest 
counterexample seen so far by $A_{i}$.
Let ${\setone{i}}(k)$ denote the set $\setone{i}$ right after the $k$-th
example query posed by $A_{i}$, and similarly for $\settwo{i}(k)$.
For all $k\geq 0$, we have that $|\settwo{i}(k)|\leq|\setone{i}(k)|\leq p(|l_i|,|x|)$. 
Since all responses to queries are as required 
by $A_{i}$, if $A_{i}$ halts 
after polynomially many polynomial size 
queries, the same happens with $A'_{i}$, which 
returns a hypothesis $\settwo{i}(k)$ equivalent to the target $l_i$. 
\end{proof}

 \section{Proofs for Section~\ref{sec:reasoning}}
 


\subsection{Alternative proof of Lemma~\ref{lem:compl-conj-el-lit}}
We provide an alternative proof of Lemma~\ref{lem:compl-conj-el-lit}.
Reasoning about the satisfiability of conjunctions of \EL literals can be delegated to the reasoning 
in \ELPP~\cite{BBL-EL}. We define the function $\tau$ from the set of \EL 
literals to the set of subsets of \ELPP CBox axioms as follows:
\[
\begin{array}{lcl}
\tau(A(a)) &  := & \{ \{a\} \sqsubseteq A \}\\
\tau(\lnot A(a)) &  := & \{ \{a\} \sqcap A \sqsubseteq \bot \}\\
\tau(C \sqsubseteq D) &  := &  \{ C \sqsubseteq D \}\\
\tau(\lnot (C \sqsubseteq D)) &  := &  \{ \{f\} \sqsubseteq C , \{f\} \sqcap D \sqsubseteq \bot \} \hfill , \text{where $f$ is a fresh individual} \\
\tau(r(a,b)) &  := &  \{ \{a\} \sqsubseteq \exists r. \{b\} \}\\
\tau(\lnot r(a,b)) &  := &   \{ \{a\} \sqcap \exists r. \{b\} \sqsubseteq \bot \}
\end{array}
\]
Now, 
given a non-empty set $L$  of \EL   literals, we define:
\[
\begin{array}{lcl}
  \tau^\bullet( L) & := &  \bigcup_{\elaxiomlit\in L}\tau(\elaxiomlit) 
\end{array}
\]
The semantics of the new syntax are: $(\{a\})^\Imc =  \{a^\Imc\}$ and $(\bot)^\Imc = \emptyset$.

An analogous reduction is presented in the proof of~\cite[Lemma~2.1]{Thost2015LTLOE}.
It is routine to show the next proposition.
\begin{proposition}\label{prop:elpp-onto-compl}
Let $L$ be a non-empty set of \EL literals. Then, $L$ is satisfiable iff $\tau^\bullet(L)$ is a satisfiable \ELPP ontology.
\end{proposition}

Clearly, $\tau^\bullet(L)$ can be computed in time polynomial in the size of $L$.
Moreover, it has been shown in~\cite{BBL-EL} that deciding the satisfiability of an \ELPP ontology can be done in polynomial time. The lemma follows.
\LemmaComplConjElLit*

\subsection{Deciding the satisfiability of \EL formulas is \NP-complete}
\label{sec:sat-el-form-compl}

The fact that the satisfiablity of \EL formulas is \NP-complete 
is a consequence of the fact that one can separate 
the satisfiability test into two independent tests: one for 
the DL dimension and one for the propositional dimension~\cite{Baader:2012:LOD:2287718.2287721,DBLP:conf/ijcai/BorgwardtT15}.
In more details, the 
 \emph{propositional 
abstraction} $\prop{\alpha}$ of an \EL formula $\alpha$ is  
the result of replacing each \EL axiom in  $\varphi$ by 
a propositional variable such that there is a $1:1$ relationship 
between the \EL axioms $\elaxiom$ occurring in $\alpha$ and the 
propositional variables $p_{\elaxiom}$ used for the abstraction. 

Given an \EL  formula $\alpha$, we say that a propositional model \Mmc
of $\prop{\alpha}$, defined as the set of variables 
evaluated to true in the model, is \emph{$\consistent{\alpha}$}
if the following formula is satisfiable $$\bigwedge_{p_{\elaxiom}\in \Mmc} {\elaxiom} \wedge 
\bigwedge_{p_{\elaxiom}\in\overline{\Mmc}}\neg {\elaxiom}$$
where $\overline{\Mmc}$ is $\{p_{\elaxiom}\mid \elaxiom \text{ is an \EL axiom in }
\alpha \}\setminus\Mmc$. 
We now formalise the connection between 
\EL formulas and their propositional abstractions with consistent models, which is an adaptation of the 
results obtained for other \EL extensions.
\begin{restatable}{proposition}{PropConnection}\cite{Baader:2012:LOD:2287718.2287721,DBLP:conf/ijcai/BorgwardtT15}\label{prop:ltl} 
An \EL  formula $\alpha$ is satisfiable if, and only if, 
$\prop{\alpha}$ is satisfiable by an $\consistent{\alpha}$ model.  
\end{restatable} 

To show that \EL formulas enjoy the polynomial size model property,
we are going to use the notion of a canonical model for sets of \EL axioms~\cite{DBLP:journals/jsc/LutzW10,dlhandbook}.
For completeness of our results and for convenience of the reader,  
we add a definition of a canonical model for a set \Omc of \EL axioms. 
We assume w.l.o.g.~that the set \Omc of \EL axioms
only contains axioms of the form $A_1 \sqcap  A_2\sqsubseteq B$, $\exists r.A \sqsubseteq B$, 
$A \sqsubseteq \exists r.B$, $A(a)$ or $r(a,b)$. 
In the following we write $\Omc\models \exists r.B(a)$ meaning for all \EL interpretations 
$\Imc$, if \Imc satisfies \Omc then $a^\Imc\in (\exists r.B)^\Imc$. 
Denote with $\NI(\Omc)$ the set $\{a\in\NI\mid a \text{ occurs in } \Omc\}$. 
 
\begin{definition}[Canonical model]\cite{Thost2015LTLOE}[Definition 5.1]
Let \Omc be a set of \EL axioms and let
$
\Delta^{\Imc_\Omc}_{\sf u}=\{c_A\mid A\in\NC(\Tmc)\cup \{\top\}\}.
$
The canonical model for \Omc is
\[
\begin{array}{ll@{\hspace*{1.5em}}ll}
\Delta^{\Imc_\Omc}&=\NI(\Omc)\cup\Delta^{\Imc_\Omc}_{\sf u},\\
a^{\Imc_\Omc}&=a,\\
A^{\Imc_\Omc}&=\{ a\in\NI(\Omc)\mid \Omc\models A(a)\}\cup 
\{c_B\in\Delta^{\Imc_\Omc}_{\sf u}\mid \Omc\models B\sqsubseteq A\},  \\
r^{\Imc_\Omc}&=\{ (a,b) \mid r(a,b)\in \Omc\}\cup 
\{(a,c_B)\in \NI(\Omc)\times\Delta^{\Imc_\Omc}_{\sf u}\mid \Omc\models \exists r.B(a)\}
\cup \\
&\quad\{(c_A,c_B)\in \Delta^{\Imc_\Omc}_{\sf u}\times\Delta^{\Imc_\Omc}_{\sf u}\mid 
\Omc\models A\sqsubseteq\exists r.B\},
\end{array}
\] 
for all $a\in \NI$, $A\in\NC$, $r\in\NR$ (in fact we only define the canonical 
model for concept/role/individual names occurring in \Omc). 
\end{definition} 
Given a concept $C$ and a set \Omc of \EL axioms, the canonical model of $C$ and \Omc is 
defined in the same way, except that we add the axioms
$A_C \sqsubseteq C$, $C \sqsubseteq A_C$ and $A_C(a_C)$, 
where $A_C$, $a_C$ is a concept and individual name, resp., 
used to encode that the extension of $C$ in the model is not empty,  
and then transform all axioms in the normal form described above.

We now show that \EL formulas enjoy the polynomial size model property.
\begin{lemma}\label{lem:el-form-polysize}
  \EL formulas enjoy the polynomial size model property. 
\end{lemma}
\begin{proof}
Let $\alpha$ be an \EL formula. 
By Proposition~\ref{prop:ltl}, $\alpha$ is satisfiable if, and only if, 
$\prop{\alpha}$ is satisfiable by an $\consistent{\alpha}$ model.  
Assume $\alpha$ is satisfiable. Then, there is an $\consistent{\alpha}$ model \Mmc. 
We are going to use $\Mmc$ to define a model 
for $\alpha$.  Since \Mmc is an $\consistent{\alpha}$ model 
the following formula is satisfiable    
    $$\Phi_\Mmc=\bigwedge_{p_{\elaxiom}\in \Mmc}  {\elaxiom} \wedge 
\bigwedge_{p_{\elaxiom}\in\overline{\Mmc}}\neg  {\elaxiom}$$
   where $\overline{\Mmc}$ is as defined in Section~\ref{sec:reasoning}. 
   
   Denote by $\Imc_{\alpha_\Mmc}$ the canonical model of 
    the set of \EL axioms occurring in $\alpha$ with $p_\alpha\in\Mmc$.  
For each \EL axiom $\elaxiom$ of the form $C\sqsubseteq D$ 
with $p_{\elaxiom}\not\in\Mmc$, 
let $\Imc_{C,\alpha_\Mmc}$ be the canonical model of $C$ and $\alpha_\Mmc$. 
Assume that the domains of $\Imc_{\alpha_\Mmc}$ and each such $\Imc_{C,\alpha_\Mmc}$ 
are pairwise disjoint. 
We define  $\Imc_\alpha$ as 
$$\Imc_{\alpha_\Mmc}\cup\bigcup_{p_{C\sqsubseteq D}\in\overline{\Mmc} } \Imc_{C,\alpha_\Mmc}$$ 
One can show with an inductive argument that $\Imc_\alpha$ is a model of $\alpha$, 
where we use $\Imc_{C,\alpha_\Mmc}$ to satisfy the negation of \EL axioms of the form 
$C\sqsubseteq D$. 
The fact that $\Imc_{C,\alpha_\Mmc}$ indeed violates $C\sqsubseteq D$ 
follows from the construction 
of  $\Imc_{C,\alpha_\Mmc}$, which only adds implied concepts, 
and the fact that $\Phi_\Mmc$ is satisfiable, meaning that $D$ is not implied 
and therefore the root of $\Imc_{C,\alpha_\Mmc}$ is not in $D^\Imc_{C,\alpha_\Mmc}$.
The fact that the negation of \EL axioms $\elaxiom$ of the form 
$A(a)$ and $r(a,b)$ with $\elaxiom\in\overline{\Mmc}$ is satisfied,
follows from the construction of $\Imc_{\alpha}$. 
\end{proof}

This is an immediate consequence of Lemma~\ref{lem:el-form-polysize}.
\begin{theorem}\label{EL-formulas-sat}
  Deciding the satisfiability of \EL formulas is \NP-complete.
\end{theorem}

\subsection{Proof of Lemma~\ref{lem:poly-size}}
%
%


We now extend the result of Theorem~\ref{EL-formulas-sat} to \ELK formulas. 
To show that satisfiability of \ELK formulas is \NP-complete, 
we establish that \ELK enjoys the polynomial size  
model property: if an \ELK formula $\varphi$  has a model then it has a model polynomial in $|\varphi|$.

\LemmaPolyModel*
  \begin{proof}
    Let $\Imf=(\Wmc,\{\Rmc_i\}_{i\in \agents})$ be an \ELK interpretation and let
    $\Imc \in \Wmc$ be an \EL interpretation. Let also $\varphi$ be an \ELK formula.
    Suppose that $\Imf,\Imc\models\varphi$.
    Because of Lemma~\ref{lem:el-form-polysize}, we can assume w.l.o.g.~that 
    each \EL interpretation $\Imc\in\Wmc$ has a size polynomial in $|\varphi|$. 

    Collect in $K(\varphi)$ the \ELK axioms of the form $\K_i \alpha$ occurring in $\varphi$.
    Define $K^{-}(\varphi) = \{\beta \in K(\varphi) \mid \Imf,\Imc\not\models \beta\}$.

    For every $\beta = \K_{\beta_1}\ldots \K_{\beta_k}\alpha \in K^{-}(\varphi)$ (where $k \geq 1$, and $\alpha$ is an \EL axiom), we have $\Imf,\Imc\not\models \beta$ by definition. So there are $k$ (not necessarily distinct) \ELK interpretations $\Imc_{\beta_1}, \ldots, \Imc_{\beta_k}$ in $\Wmc$ such that:
(1)~$(\Imc, \Imc_{\beta_1}) \in \Rmc_{\beta_1}$;
(2)~$(\Imc_{\beta_{i}}, \Imc_{\beta_{i+1}}) \in \Rmc_{\beta_{i+1}}$ for $1 \leq i < k$;
(3)~$\Imf,\Imc_{\beta_k} \not\models\alpha$.

  Now, define $\Wmc' = \{\Imc\} \cup \{ \Imc_{\beta_1}, \ldots, \Imc_{\beta_k} \mid
  \beta = \K_{\beta_1}\ldots \K_{\beta_k}\alpha \in K^{-}(\varphi) \}$.
  Also, for each agent $i \in \agents$, build the relation $\Rmc'_i$ such that  
  $(\Jmc, \Jmc') \in \Rmc'_i$ iff $\Jmc, \Jmc' \in \Wmc'$ and $(\Jmc, \Jmc') \in \Rmc_i$.
  Finally, let $\Imf'$ be the \ELK interpretation $(\Wmc', \{\Rmc'_i\}_{i \in \agents})$.


  \medskip
  To conclude the proof, we claim that
  (i)~the size of $\Imf'$ is polynomial in the size of $\varphi$, and (ii)~$\Imf, \Imc \models \varphi$.

 To see why (i)~holds, observe that by construction, $\Wmc'$ contains $\Imc$ and at most one extra interpretation for every modal operator of each element of $K^{-}(\varphi)$. Hence, the number of \EL interpretations in $\Wmc'$ is linear in the size of $\varphi$. 
By assumption each $\Imc\in\Wmc$ is polynomial in $|\varphi|$. Hence, the size of $\Imf'$ is polynomial in the size of $\varphi$.
To establish (ii), it suffices to show that for every \ELK axiom $\beta$ occurring in $\varphi$ we have $\Imf', \Imc \models \beta$ iff $\Imf, \Imc \models \beta$.
If $\beta$ is an \EL axiom (not in $K(\varphi)$), it is immediate. 
    If $\beta = \K_{\beta_1}\ldots \K_{\beta_k}\alpha \in K(\varphi) \setminus K^{-}(\varphi)$, we have by definition that $\Imf,\Imc \models \K_{\beta_1}\ldots \K_{\beta_k}\alpha$. Since by construction we have $\Rmc'_{i} \subseteq \Rmc_{i}$ for every agent $i \in \agents$, we also have $\Imf', \Imc \models \K_{\beta_1}\ldots \K_{\beta_k}\alpha$.
    If $\beta = \K_{\beta_1}\ldots \K_{\beta_k}\alpha \in K^{-}(\varphi)$, we have by definition that $\Imf,\Imc \not\models \K_{\beta_1}\ldots \K_{\beta_k}\alpha$. Moreover, there are $\Imc_{\beta_1}, \ldots, \Imc_{\beta_k} \in \Wmc'$ satisfying the properties 1--3 in the construction above. Hence, we also have $\Imf', \Imc \not\models \K_{\beta_1}\ldots \K_{\beta_k}\alpha$.
  \end{proof}

\subsection{Proof of Proposition~\ref{prop:corr-conj-EL}}
 
For every $\sigma = a_1 \ldots a_k \in \agents^*$ we note $\Rmc_\sigma = \Rmc_{a_1} \circ \ldots \circ \Rmc_{a_k}$
and $\KK_\sigma = \K_{a_1} \ldots \K_{a_k}$.
The empty sequence is noted $\epsilon$, and we have $\Rmc_\epsilon = Id$, where $Id$ is the identity relation, and $\KK_\epsilon \omega = \omega$.

The two following lemmas, which are instrumental in the proof of Proposition~\ref{prop:corr-conj-EL}, are simple consequences of well-known properties of the modal system $S5$~\cite[p.~58]{HughesCresswell}.
\begin{lemma}\label{lem:flat-equiv}
  If $\varphi^\flat$ is an \ELK formula,
$\varphi$ and $\varphi^\flat$ are equivalent.
\end{lemma}
\begin{proof}
Given an \ELK interpretation $(\Wmc,\{\Rmc_i\}_{i\in \agents})$, 
   we have that $\Rmc_i$ is an equivalence relation, for every $i \in \agents$. 
  In particular it is transitive and reflexive, and we have $\Rmc_i\circ \Rmc_i = \Rmc_i$. 
  Thus, for every pointed interpretation $(\Imf,\Imc)$, for every agent $i \in \agents$, 
  and for every formula $\psi$, we have $(\Imf,\Imc) \models \K_i\K_i\psi$ iff $(\Imf,\Imc) \models \K_i\psi$.
\end{proof}
\begin{lemma}\label{lem:subword}
Let $(\Wmc,\{\Rmc_i\}_{i\in \agents})$ be an \ELK interpretation.
For all $\sigma\in\agents^*$ and $\sigma'\in\agents^*$, 
if $\sigma$ is a subword of $\sigma'$ then $\Rmc_{\sigma} \subseteq \Rmc_{\sigma'}$.
\end{lemma}
\begin{proof}
  For every $i \in \agents$, because $\Rmc_i$ is an equivalence relation, and thus reflexive, we have $Id \subseteq \Rmc_i$. We can thus insert arbitrarily additional agents in the occurence of $\sigma$ appearing
  in the right-hand side of $\Rmc_{\sigma} \subseteq \Rmc_{\sigma}$, and obtain the result.
\end{proof}

\PropCorrConjEL*
\begin{proof}
Since, by Lemma~\ref{lem:flat-equiv}, $\varphi^\flat$ and $\varphi$ are equisatisfiable (in fact   equivalent), 
we prove the property for $\varphi^\flat$.
For the ($\Leftarrow$) direction, suppose (1) holds. 
Since $\Rmc_i$ is reflexive for all $i \in \agents$, every model satisfying $\varphi^\flat$ 
must satisfy $\omega_0 \land \bigwedge \{ \omega \mid \KK_{\sigma} \omega \in \varphi^\flat \}$. 
Since it is not \EL satisfiable, there cannot be an \ELK interpretation satisfying $\varphi^\flat$ either. 
Suppose (2) holds. For some $\lnot \KK_{\sigma}\omega \in \varphi^\flat$, we have that 
$\psi = \lnot\omega \land \bigwedge \{ \omega' \mid \KK_{\sigma'}\omega' \in \varphi^\flat, 
\text{ and } \sigma \text{ is a subword of } \sigma' \}$ is not \EL satisfiable.
Suppose towards contradiction that there exist
an \ELK interpretation $\Imf=(\Wmc,\{\Rmc_i\}_{i\in \agents})$ and an \EL interpretation 
$\Jmc \in \Wmc$ such that $(\Imf, \Jmc) \models \varphi^\flat$. It means that 
$(\Imf, \Jmc) \models \lnot \KK_{\sigma}\omega$, that is, there is an \EL 
interpretation $\Jmc$ such that $(\Jmc, \Jmc') \in \Rmc_\sigma$ and $(\Imf, \Jmc') \models \lnot \omega$.
From Lemma~\ref{lem:subword}, for every $\KK_{\sigma'}\omega' \in \varphi^\flat$, if 
$\sigma$ is a subword of $\sigma'$, then $\Rmc_{\sigma} \subseteq \Rmc_{\sigma'}$. 
Hence,  $(\Imf, \Jmc') \models \psi$, which is a contradiction because $\psi$ 
is not \EL satisfiable.

  ($\Rightarrow$) Assume that none of (1) and (2) hold. We must show that $\varphi^\flat$ is satisfiable.
It suffices to build an \ELK interpretation $\Imf=(\Wmc,\{\Rmc_i\}_{i\in \agents})$ for $\varphi^\flat$:
\begin{itemize}
\item (Construction of $\Wmc$) $\Wmc$ contains an \EL interpretation $\Jmc_0$ satisfying $\omega_0 \land 
\bigwedge \{ \omega \mid \KK_{\sigma} \omega \in \varphi^\flat \}$. Such an interpretation 
exists because (1) does not hold.
For each $\lnot \KK_{\sigma}\omega \in \varphi^\flat$, where $\sigma = a_1 \ldots a_k$, $\Wmc$ 
also contains an interpretation $\Jmc^\sigma_k$ satisfying
$\lnot\omega \land \bigwedge \{ \omega' \mid \KK_{\sigma'}\omega' \in \varphi^\flat, 
\text{ and } \sigma \text{ is a subword of } \sigma' \}$.
Such an interpretation exists because (2) does not hold.
Still for each $\lnot \KK_{\sigma}\omega \in \varphi^\flat$, $\Wmc$, where $\sigma = a_1\ \ldots a_k$, 
for each $1 \leq i < k$,
$\Wmc$ also contains an interpretation $\Jmc^\sigma_i$ satisfying
$\bigwedge \{ \omega' \mid \KK_{\sigma'}\omega' \in \varphi^\flat, \text{ and } a_1 \ldots
a_i \text{ is a subword of } \sigma' \}$. Such interpretations exist because (1) does not hold.
$\Wmc$ does not contain any more \EL interpretations.
\item (Construction of $\Rmc_i$, $i\in\agents$) For every $\lnot \KK_{\sigma}\omega \in \varphi^\flat$, 
where $\sigma = a_1 \ldots a_k$, for every $1 \leq i \leq k$, let $(\Jmc^\sigma_{i-1}, \Jmc^\sigma_i) 
\in \Rmc'_{a_i}$, where $\Jmc^\sigma_{0} = \Jmc_{0}$. For every $i \in \agents$, let $\Rmc_i$ 
be the equivalence closure of $\Rmc'_i$.
\end{itemize}
It is routine to check that $\Imf, \Jmc_0 \models \varphi^\flat$.
\end{proof}